\documentclass{article}
\usepackage{ijcai25}

\usepackage{times}
\usepackage{soul}
\usepackage{url}
\usepackage[hidelinks]{hyperref}
\usepackage[utf8]{inputenc}
\usepackage[small]{caption}
\usepackage{graphicx}
\usepackage{amsmath}
\usepackage{amsthm}
\usepackage{booktabs}
\usepackage[switch]{lineno}

\usepackage{amsxtra, amsfonts, amssymb, amstext, mathtools}
\usepackage{nicefrac}
\usepackage{xspace}
\usepackage[usenames,dvipsnames]{xcolor}
\usepackage[algo2e,ruled,vlined,linesnumbered]{algorithm2e}
\SetArgSty{} 
\usepackage{wrapfig}
\usepackage{enumitem}
\usepackage[noabbrev, nameinlink]{cleveref}
\usepackage{thm-restate}

\crefname{assumption}{assumption}{assumptions}
\crefname{observation}{observation}{observations}

\allowdisplaybreaks[4]
\newtheorem{theorem}{Theorem}
\newtheorem{lemma}[theorem]{Lemma}

\newcommand{\spea}{SP\-EA\-2\xspace}
\newcommand{\nsga}{NS\-GA-II\xspace}
\newcommand{\sigDist}{$\sigma$-distance\xspace}
\newcommand{\sigDists}{$\sigma$-distances\xspace}

\newcommand{\omm}{\textsc{OMM}\xspace}

\newcommand{\R}{\ensuremath{\mathbb{R}}}

\newcommand{\N}{\ensuremath{\mathbb{N}}} 

\newcommand{\E}{\ensuremath{\mathbb{E}}}

\let\originalleft\left
\let\originalright\right
\renewcommand{\left}{\mathopen{}\mathclose\bgroup\originalleft}
\renewcommand{\right}{\aftergroup\egroup\originalright}

\DeclarePairedDelimiter{\floor}{\lfloor}{\rfloor}
\DeclarePairedDelimiter{\ceil}{\lceil}{\rceil}

\title{Proven Approximation Guarantees in Multi-Objective Optimization: \\SPEA2 Beats NSGA-II}

\author{
Yasser Alghouass$^1$\and
Benjamin Doerr$^2$\and
Martin~S. Krejca$^2$\And
Mohammed Lagmah$^1$\\
\affiliations
$^1$École Polytechnique, Institut Polytechnique de Paris\\
$^2$Laboratoire d'Informatique (LIX), CNRS, École Polytechnique, Institut Polytechnique de Paris\\
\emails
\{firstname.lastname\}@polytechnique.edu
}

\begin{document}
{\sloppy

\maketitle

\begin{abstract}
  Together with the NSGA-II and SMS-EMOA, the \emph{strength Pareto evolutionary algorithm 2} (SPEA2) is one of the most prominent dominance-based multi-objective evolutionary algorithms (MOEAs). Different from the NSGA-II, it does not employ the crowding distance (essentially the distance to neighboring solutions) to compare pairwise non-dominating solutions but a complex system of $\sigma$-distances that builds on the distances to all other solutions.
  In this work, we give a first mathematical proof showing that this more complex system of distances can be superior. More specifically, we prove that a simple steady-state SPEA2 can compute optimal approximations of the Pareto front of the OneMinMax benchmark in polynomial time. The best proven guarantee for a comparable variant of the NSGA-II only assures approximation ratios of roughly a factor of two, and both mathematical analyses and experiments indicate that optimal approximations are not found efficiently.
\end{abstract}

\section{Introduction}
\label{sec:introduction}

Many optimization problems in practice consist of several conflicting objectives. One common approach for such problems is to compute a set of solutions witnessing the Pareto front (or a sufficiently diverse subset thereof) and then let a human decision maker select the final solution. For such multi-objective optimization problems, evolutionary algorithms with their population-based nature are an obvious choice, and in fact, such \emph{multi-objective evolutionary algorithms (MOEAs)} are among the most successful algorithms~\cite{CoelloLV07,ZhouQLZSZ11}.

Most research on randomized search heuristics is empirical, and much fewer theoretical works exist~\cite{NeumannW10,AugerD11,Jansen13,ZhouYQ19,DoerrN20}. However, theoretical analyses of MOEAs, usually proving runtime guarantees and from this aiming at a deeper understanding of these algorithms, exist for more than twenty years~\cite{LaumannsTDZ02,Giel03,Thierens03}. This field has made a huge step forward recently with the first mathematical runtime analysis of the \nsga \cite{ZhengLD22,ZhengD23aij}, the most prominent MOEA. This work was quickly followed up by more detailed analyses of the \nsga~\cite{BianQ22,DoerrQ23tec,DoerrQ23LB,DoerrQ23crossover,DangOSS23gecco,CerfDHKW23,DangOS24,DengZLLD24,ZhengD24approx,ZhengD24many,DoerrIK25} and by analyses of other prominent MOEAs… such as the NSGA-III~\cite{WiethegerD23,OprisDNS24,DengZD25,Opris25}, SMS-EMOA~\cite{BianZLQ23,ZhengD24,LiZD25}, and SPEA2 \cite{RenBLQ24}.

Interestingly, these results show very similar performance guarantees for these algorithms (which agree with the results known for the classic (G)SEMO analyzed in the early theoretical works on MOEAs).
Also, \cite{RenBLQ24} provide a mathematical framework allowing to prove comparable runtime bounds for several MOEAs.
The sole outlier so far is the result~\cite{ZhengD24many} showing that the \mbox{\nsga} has enormous difficulties for three or more objectives (see~\cite{DoerrKK24arxiv} for a second such result). This problem of the \nsga can  be resolved with an additional tie-breaker~\cite{DoerrIK25} and then the same performance guarantees as known for the other algorithms hold.

In this work, we detect a second notable performance difference, namely in the ability to approximate the Pareto front. This aspect is understood much less than the time to compute the full Pareto front. The only such result for the algorithms mentioned above is \cite{ZhengD24approx}. In that work, it was argued via theoretical arguments that the classic \nsga, computing the crowding distance for all individuals and then selecting the next population, can have difficulties to approximate the Pareto front as it ignores that every removal of an individual affects the crowding distance of other solutions. However, for the algorithm variant that removes the individuals sequentially and updates the crowding distances after each removal (as suggested already in the little known work \cite{KukkonenD06}), a $2$-approximation result was shown for the optimization of the bi-objective OneMinMax benchmark. This is the only approximation guarantee so far for one of the classic algorithms named above.

\paragraph{Our contribution.}
In this work, we study the approximation ability of the \emph{simple Pareto evolutionary algorithm 2} \cite{ZitzlerLT01} (SPEA2). Like the \nsga, the \spea is a dominance-based algorithm, that is, its first priority is to keep non-dominated solutions. As a tie-breaker for removing solutions, like the \nsga, it regards how close incomparable solutions are and first removes those with other solutions nearby, thus aiming for an evenly spread population. We discuss the precise differences of the distance measures used by the \nsga and \spea later in this work, and remark for now only that the \spea takes into account the distances to all other solutions (in the order from near to far), whereas the crowding distance of the \nsga only accounts for the distances to the two neighboring solutions in each objective.

As we shall show in this work, with its more complex distance measure, the SPEA2 is able to compute much better approximations of the Pareto front. Taking again the OneMinMax benchmark as example, we prove that a simple version of the SPEA2 finds an optimal approximation to the Pareto front in polynomial time, more precisely, expected time $O(\mu^2 \log(\mu) n \log(n))$ (\Cref{thm:spea_run_time_optimal_spread}), where $\mu$ is the population size of the \spea and $n$ is the problem size of the benchmark.
We contrast this result by showing that there are states in the \nsga that are close to being an optimal approximation, but the algorithm takes nonetheless with overwhelming probability a super-polynomial time to get to the optimal approximation (\Cref{thm:nsga-counterexample}).
This shows that the \nsga can struggle immensely to compute an optimal approximation whereas the \spea does not have this problem.

These results give a strong evidence for the hypothesis that the more complex way of measuring the distance between solutions of the \spea is worth the additional complexity and leads to significantly stronger approximation abilities.

\section{Preliminaries}
\label{sec:preliminaries}

Denote the natural numbers by~$\N$ (with~$0$) and the reals by~$\R$.
For all $m, n \in \N$, let $[m .. n] \coloneqq [m, n] \cap \N$ and $[n] \coloneqq [1 .. n]$.

For all $n \in \N_{\geq 1}$, we call $x \in \{0, 1\}^n$ an \emph{individual}.
We call a multi-set of individuals a \emph{population}, and we use standard set operations for populations, which extend naturally to multi-sets.
We consider pseudo-Boolean maximization of bi-objective functions $f\colon \{0, 1\}^n \to \R^2$, which map individuals to \emph{objective values}.
For each $x \in \{0, 1\}^n$ and $i \in [2]$, we denote the objective value~$i$ of~$x$ by $f_i(x)$.

We compare objective values $u, v \in \R$ via the \emph{dominance} partial order.
We say \emph{$u$ weakly dominates~$v$} (written $u \succeq v$) if and only if $u_1 \geq v_1$ and $u_2 \geq v_2$, and~$u$ \emph{strictly} dominates~$v$ if and only if at least one of these inequalities is strict.
If and only if neither $u \succeq v$ nor $v \succeq u$, we say~$u$ and~$v$ are \emph{incomparable}.
Given a bi-objective function, we extend this notion to all individuals by implicitly referring to their objective value.

For a (multi-)set~$P$ of individuals and a bi-objective function~$f$, we say that an individual $x \in P$ is \emph{non-dominated} if and only if there is no $y \in P$ that strictly dominates~$u$.
We call each non-dominated individual of $\{0, 1\}^n$ \emph{Pareto-optimal} and its objective value a \emph{Pareto optimum}.
Last, we call the set of all Pareto optima the \emph{Pareto front (of~$f$)}.

We use traditional set notation for both normal sets (with unique elements) and multi-sets.
The union of multi-sets does not remove any duplicates, and the cardinality of a multi-set accounts for all duplicates.
However, if we apply a function to multi-set, then the result is a normal function, that is, the function values are not duplicated.
Note that populations are multi-sets whereas sets of objective values are normal sets.

\paragraph{The OneMinMax Benchmark.}
The OneMinMax (\omm) benchmark, introduced by \cite{GielL10}, is a bi-objective function that aims at maximizing the number of ones and the number of zeros of an individual.
Formally, for all $x \in \{0, 1\}^n$, we have
\begin{align*}
  \textstyle
  \omm(x) = \Bigl(\sum\nolimits_{i \in [n]} x_i, \sum\nolimits_{i \in [n]} (1 - x_i)\Bigr) .
\end{align*}

Since the objectives are the inverse of each other, no individual strictly dominates another one.
Thus, the Pareto front is the set of all objective values, namely $\{(i, n - i) \mid i \in [0 .. n]\}$, which has a size of $n + 1$.

We call the objective values $(0, n)$ and $(n, 0)$ the \emph{extreme} objective values, since they feature the maximum and minimum possible value in each of their objectives.

\paragraph{Optimal spread.}
We aim at approximating the Pareto front on \omm algorithmically via a population.
To this end, let~$P$ be a population of size $\mu \coloneqq |P| \in [2 .. n]$, and let $\bigl(\omm_1(x)\bigr)_{x \in P} \eqqcolon (v_i)_{i \in [\mu]}$ be the first objective of each individual in~$P$ in ascending order.
Furthermore, assume that $v_1 = 0$ and that $v_\mu = n$, that is, that the extreme objective values are witnessed by~$P$.
We say that~$P$ computes an \emph{optimal spread} on the Pareto front of \omm if and only if for all $i \in [\mu - 1]$, we have $v_{i + 1} - v_i \in \{\floor{\frac{n}{\mu - 1}}, \ceil{\frac{n}{\mu - 1}}\}$.
In other words, the distance between two neighboring objective-values is as close as possible to being equidistant.

In order to more concisely describe an optimal spread, let
\begin{align}
  \label{eq:alpha_and_beta}
   & \textstyle\alpha \coloneqq \floor{\frac{n}{\mu - 1}} \textrm{ and }
  \beta \in [0 .. \mu - 1]                                                 \\
  \notag
   & \textrm{such that } \alpha \beta + (\alpha + 1)(\mu - 1 - \beta) = n,
\end{align}
where we note that~$\beta$ is uniquely determined.

\section{The \spea and NSGA-II Algorithms}
\label{sec:algorithms}

The \emph{strength Pareto evolutionary algorithm~2}~\cite{ZitzlerLT01} (\spea, \Cref{alg:spea2}) and the \emph{non-dominated sorting genetic algorithm~II}~\cite{DebPAM02} (\nsga, \Cref{alg:NSGA}) are both popular population-based multi-objective optimization heuristics.
Since we analyze the \spea in more depth in this paper, we explain it in more detail below in \Cref{sec:algorithms:spea}, noting that both algorithms act identically in broad parts on the \omm benchmark.
We outline the most important parts of the \nsga in \Cref{sec:algorithms:nsga}, referring for more details to the original paper by \cite{DebPAM02}.

\subsection{The \spea}
\label{sec:algorithms:spea}

We regard an equivalent formulation of the \spea that was proposed in the extension of~\cite{WiethegerD24} (currently only on arXiv). The \spea (\Cref{alg:spea2}) maximizes a given bi-objective problem $f\colon \{0, 1\}^n \to \R^2$ by iteratively refining  a multi-set of individuals (the \emph{parent population}) of given population size $\mu \in \N_{\geq 1}$.
In each iteration, the algorithm generates a user-defined amount $\lambda \in \N_{\geq 1}$ of new individuals (the \emph{offspring population}) via a process called \emph{mutation}.
Afterward, the \spea selects the~$\mu$ most promising individuals among the parent and the offspring population, to be used in the next iteration.
The basis for this new population are all non-dominated individuals.
If this number is greater than~$\mu$, then the algorithm removes individuals based on a clustering technique called \emph{\sigDists}.
This is the predominant case in our analysis in \Cref{sec:theory-spea}, as the \omm benchmark only features non-dominated objective values.

If, instead, the number of non-dominated individuals is less than~$\mu$, then the \spea adds dominated individuals, based on their \emph{strength-based indicator}.
Last, if the number of non-dominated individuals is exactly~$\mu$, the algorithm proceeds immediately with the next iteration.

\paragraph{Mutation.}
We consider two types of mutation, each of which is given a \emph{parent} $x \in \{0, 1\}^n$ and generates a \emph{new} individual $y \in \{0, 1\}^n$ (the \emph{offspring}).

\emph{1-bit mutation} copies~$x$ and flips exactly one bit at position $i \in [n]$ chosen uniformly at random.
That is, we have $y_i = 1 - x_i$, and for all $j \in [n] \smallsetminus \{i\}$, we have $y_i = x_i$.

\emph{Standard bit mutation} copies~$x$ and flips each position independently with probability~$\frac{1}{n}$.
That is, for all $i \in [n]$, we have $\Pr[y_i = 1 - x_i] = \frac{1}{n}$ and $\Pr[y_i = x_i] = 1 - \frac{1}{n}$, independent from all other random choices.

\paragraph{The \sigDists.}
Given a population~$P$ of size $s \in \N_{\geq 1}$ of non-dominated individuals, the \sigDists are based on a function $\sigma\colon \{0, 1\}^n \to \R^{s - 1}$ that assigns each individual $x \in P$ a vector that contains the Euclidean distances of the objective value of~$x$ to those of all other individuals in~$P$ in ascending order (breaking ties arbitrarily), which we call the \sigDist of~$x$.
We follow the convention of \cite{ZitzlerLT01} and denote for all $i \in [s - 1]$ the distance of~$x$ to its $i$-closest neighbor by~$\sigma^i_x$.

In \cref{line:sigDists}, individuals are removed with respect to the lexicographic ascending order of their \sigDist, breaking ties uniformly at random, and removing the individuals at the beginning of this order.
In other words, individuals that are too close to other individuals are removed first.
We remark that the \sigDists are updated after each removed individual.

We note that for \omm, we always have $s = \mu + \lambda$, as all individuals are Pareto-optimal.
Moreover, we note that for the removal based on \sigDists, only the relative order matters.
Since \omm has complementary objectives, it is thus sufficient to consider a distance based on a single objective that has the same monotonicty as the original \sigDists.
We focus in our analysis on the first objective, that is, the number of~$1$s in an individual, for the \sigDists.

\begin{algorithm2e}[t]
  \caption{
    The strength Pareto evolutionary algorithm~2 \protect\cite{ZitzlerLT01} (\spea) with parent population size $\mu \in \N_{\geq 1}$, and offspring population size $\lambda \in \N_{\geq 1}$, maximizing a given bi-objective function $f\colon \{0, 1\}^n \to \R^2$.
  }
  \label{alg:spea2}
  $t \gets 0$\;
  $P_t \gets \lambda$ independent samples from $\{0, 1\}^n$ with replacement, chosen uniformly at random (u.a.r.)\;
  \While{termination criterion not met}{
  $Q_t \gets \emptyset$\;
  \For{$i \in [\lambda]$}{
    $x \gets$ individual in~$P_t$ chosen u.a.r.\;
    $y \gets$ mutate~$x$ (see \Cref{sec:algorithms})\;
    $Q_t \gets Q_t \cup \{y\}$\;
  }
  $P_{t + 1} \gets$ non-dominated individuals in $P_t \cup Q_t$\;
  \If{$|P_{t + 1}| > \mu$}{
  iteratively remove an individual in~$P_{t + 1}$ with the smallest \sigDist until $|P_{t + 1}| = \mu$\;\label{line:sigDists}
  }
  \ElseIf{$|P_{t + 1}| < \mu$}{\label{line:fillUp}
  iteratively add an individual from $(P_t \cup Q_t) \smallsetminus P_{t + 1}$ to~$P_{t + 1}$ with the smallest strength-based indicator until $|P_{t + 1}| = \mu$\;
  }
  $t \gets t + 1$\;
  }
\end{algorithm2e}

\begin{algorithm2e}[t]
  \caption{
    The non-dominated sorting genetic algorithm~II \protect\cite{DebPAM02} (\nsga) with population size $N\in\N_{\geq 1}$, maximizing a given bi-objective function.
  }
  \label{alg:NSGA}
  $t \gets 0$\;
  $P_t \gets N$ independent samples of $\{0, 1\}^n$ with replacement, each chosen u.a.r.\;
  \While{termination criterion not met}{
  $Q_t \gets$ offspring population of~$P_t$ of size~$N$\;
  $R_t \gets P_t \cup Q_t$\;
  $(F_j)_{j \in [r]} \gets$ partition of~$R_t$ w.r.t. non-dom. ranks\;
  $j^* \gets$ critical rank of $(F_j)_{j \in [r]}$\;
  $P_{t + 1} \gets \bigcup_{j \in [j^*]} F_j$\;
  \If{$|P_{t + 1}| > N$}{%
  iteratively remove an individual in~$P_{t + 1}$ from~$F_{j^*}$ with the smallest crowding distance in~$F_{j^*}$ until $|P_{t + 1}| = N$\;
  }
  $t \gets t + 1$\;
  }
\end{algorithm2e}

\paragraph{Strength-based indicator.}
We note that the case in which the \spea relies on the strength-based indicator never occurs on \omm, as all individuals in \omm are Pareto-optimal.
Hence, we do not explain this operation in detail but refer to the original work by \cite{ZitzlerLT01} instead.
Roughly, the strength-based indicator assigns each individual~$x$ in the combined parent and offspring population a natural number that is the sum over all individuals~$y$ that strictly dominate~$x$, adding the number of individuals that~$y$ weakly dominates.

\paragraph{Steady-state variant.}
If $\lambda = 1$, we say that the \spea is \emph{steady-state}.
When optimizing \omm, this means that a single individual in \cref{line:sigDists} is removed, as all individuals are Pareto-optimal.
The case in \cref{line:fillUp} is never executed.

\subsection{The \nsga}
\label{sec:algorithms:nsga}

Like the \spea, the \nsga (\Cref{alg:NSGA}) maximizes a given bi-objective function $f\colon \{0, 1\}^n \to \R^2$ by maintaining a parent population of $N \in \N_{\geq 1}$ individuals, from which it generates~$N$ offspring each iteration, using a mutation operator.
Afterward, it selects individuals greedily based on the number of individuals by which they are strictly dominated, starting with the non-dominated individuals.
During this phase, all individuals with an identical number are selected.
The smallest number where this leads to selecting at least~$N$ individuals in total is known as the \emph{critical rank}.
For \omm, this means that the entire combined parent and offspring population (of size~$2N$) is kept, as all individuals are not strictly dominated, thus all land in the critical rank.

Afterward, the algorithm removes individuals sequentially from the critical rank based on their \emph{crowding distance}, until~$N$ individuals remain, breaking ties uniformly at random.

\paragraph{Crowding distances.}
Given a population~$R$, the crowding distance of an individual $x \in R$ is the sum of its crowding distance \emph{per objective}.
For each objective $i \in [2]$, the crowding distance of~$x$ is based on the (normalized) distance to its two closest neighbors in objective~$i$.
That is, let $S_i = (y_i)_{i \in [|R|]}$ denote~$R$ in ascending order of objective~$i$.
If~$x$ is an extreme solution, that is, if $x = y_1$ or $x = y_{|R|}$, then its crowding distance for objective~$i$ is plus infinity.
Otherwise, if there is an $i \in [2 .. |R| - 1]$ such that $x = y_i$, then the crowding distance of~$x$ for objective~$i$ is $\bigl(f(y_{i + 1}) - f(y_{i - 1})\bigr) / \bigl(f(y_{|R|}) - f(y_1)\bigr)$.

For \omm, similar to the \spea, it is sufficient to only consider the crowding distance in the first objective, as the two objectives are complementary to each other (and the sorting for each objective just results in an inverse order, up to how ties are handled).

We remark that, different from the \sigDists in the \spea, the crowding distance is \emph{not} recomputed each time an individual is removed, which is a general flaw in the \nsga, as proven by \cite{ZhengD24approx}.
In this article, we use a stable sorting algorithm for the crowding distance computation, which is common practice.

\paragraph{Steady-state variant.}
If the \nsga only produces one offspring each iteration (regardless of~$N$), we say that it is \emph{steady-state}.
This means that in each iteration, exactly one individual is removed.
Since this effectively recomputes the crowding distance after each removal of an individual, this algorithm variant does not have the problems associated with the classic \nsga.

\section{The \spea Computes an Optimal Spread Efficiently}
\label{sec:theory-spea}

Our main result is \Cref{thm:spea_run_time_optimal_spread}, which shows that the steady-state \spea with 1-bit mutation and $\mu \leq \frac{n}{3}$ computes an optimal spread on \omm after $O(\mu^2 n \log(\mu) \log(n))$ expected function evaluations\footnote{This is asymptotically equal to the number of iterations, as each iteration generates exactly one offspring in the steady-state variant.}.
This bound is slower by a factor of $\mu \log \mu$ than the time required to find the extreme objective values (\Cref{thm:time_until_extreme_values}), which may be an artifact of our analysis.

\begin{restatable}{theorem}{speaRunTimeOptimalSpread}
  \label{thm:spea_run_time_optimal_spread}
  Consider the steady-state \spea optimizing \omm with 1-bit mutation and with $\mu \leq \frac{n}{3}$.
  Then the expected number of objective function evaluations for computing an optimal spread is $O(\mu^2 n \log(\mu) \log(n))$.
\end{restatable}

In the following, we start with a high-level overview of the proof structure.
Afterward, we provide more details, until we prove \Cref{thm:spea_run_time_optimal_spread}.

\subsection{High-Level Proof Outline for the \spea}
\label{sec:theory-spea:proofOutline}

Due to all individuals being Pareto-optimal for \omm, our analysis only concerns the case that individuals are removed due to \cref{line:sigDists} in \Cref{alg:spea2}.
For most of our analysis, we consider the steady-state variant, which implies that exactly one individual is removed.
In order to simplify notation, we only consider the first objective when comparing \sigDists, which is feasible for \omm due to the objectives being complementary, as we also mention in the section on \sigDists in \Cref{sec:algorithms}.

\paragraph{No duplicate objective values.}
First, we show that the \spea quickly reaches a state in which all individuals in the parent population have distinct objective values (\Cref{thm:time_to_distinct_archive}), which is possible because the population size~$\mu$ is strictly smaller than the size $n + 1$ of the Pareto front.
Moreover, we prove that the parent population quickly includes the extreme objective values and never loses them (\Cref{thm:time_until_extreme_values}).
Once in such a state, we change perspectives and no longer consider individuals but instead the empty intervals between the first objective values of all individuals in the parent population.
From this perspective, an optimal spread is computed once all intervals are of size $\floor{\frac{n}{\mu - 1}}$ or $\ceil{\frac{n}{\mu - 1}}$.

\paragraph{Useful invariants.}
We continue by proving that the size of all minimum length (empty) intervals (of objective values) never decreases during the run of the algorithm (\Cref{lem:increasing_smallest_interval}).
In addition, we show that the number of intervals of minimum length also never decreases (\Cref{lem:increasing_minimum_number}), which in itself is already a strong property of the \sigDists of the \spea.
We show analogously that the size of all maximum-length intervals (\Cref{lem:increasing_minimum_number}) never increases.
These statements are the foundation of our analysis, as they provide well-defined states of the algorithm from which it can only improve.

Afterward, we show that once the size of all minimum length intervals is at least~$2$, \cref{line:sigDists} reduces to the case that either the offspring or its parent is removed (\Cref{lem:easyRemoval}).
This drastically decreases the complexity of the cases to consider in the remaining analysis.
In addition, we show in \Cref{lem:small_intervals_stay_at_the_borders} that intervals of minimum length remain at the borders of the interval $[0 .. n]$ (until the minimum is increased).
This helps us locate such intervals later in the analysis.

Since the \spea features a lot of useful properties once the minimum length interval has size at least~$2$, we prove in \Cref{lem:minimum_interval_at_least_two} that such a state is reached quickly.
From thereon, we take a more detailed look about how the minimum length interval is increased.

\paragraph{Reducing the number of intervals of minimum length.}
In order to remove a minimum length interval, it is sufficient to have it neighbor an interval whose length is at least two larger.
By placing an offspring in between these two intervals such that it reduces the size of the larger interval, the size of the minimum length interval increases by one if the parent is removed.
\Cref{lem:time_to_increase_minimum_length} bounds the expected number of iterations it takes to remove one minimum length interval like this.
In a nutshell, the analysis relies on the fact that minimum length intervals tend to move to the borders of the objective space $[0 .. n]$.
Once there, the minimum distance between a minimum length interval and one with a length at least two larger does not increase.
This allows eventually to decrease the number of intervals of minimum length.

\paragraph{Combining everything.}
We conclude by estimating the time to remove all minimum length intervals of a certain size, and by afterward considering all feasible sizes.
We note that our analysis views the progress of the \spea toward an optimal spread by only considering the intervals of minimum length, which is likely a bit pessimistic, as it disregards progress made with other intervals of sub-optimal length.

\subsection{The \spea Computes an Optimal Spread Efficiently}
\label{sec:theory-spea:speaDetails}

We provide the formal details for our outline given in \Cref{sec:theory-spea:proofOutline}, following the same structure.

\paragraph{No duplicate objective values.}
We first show that the \spea never reduces the amount of Pareto optima it finds.
This is a simple and desirable general property.

\begin{restatable}{lemma}{increasingArchive}
  \label{lem:increasing_archive}
  Consider the \spea optimizing a multi-objective function~$f$, with any parent and offspring population size and with any mutation operator.
  Let $t \in \N$ such that $P_t$ only contains Pareto-optimal individuals.
  Then $|f(P_t)| \leq |f(P_{t+1})|$.
\end{restatable}

For \omm, this leads immediately to the following bound of objective-function evaluations until each individual has a unique objective value.

\begin{restatable}{theorem}{timeToDistinctArchive}
  \label{thm:time_to_distinct_archive}
  Consider the \spea optimizing \omm with either 1-bit or standard bit mutation and any parent and offspring population size.
  The expected number of \omm evaluations until the parent population contains only distinct objective values or all objective values is $O\bigl(\mu + \lambda \frac{n\log n}{1-(1-1/\mu)^\lambda}\bigr)$.
\end{restatable}

Moreover, we prove that the extreme solutions of \omm are found quickly.

\begin{restatable}{theorem}{timeUntilExtremeValues}
  \label{thm:time_until_extreme_values}
  Consider the \spea optimizing \omm with either 1-bit or standard bit mutation and any parent and offspring population size.
  The expected number of \omm evaluations until the parent population contains the extreme objective values is $O\bigl(\mu + \lambda \frac{n\log n}{1-(1-1/\mu)^\lambda}\bigr)$.

  For 1-bit mutation and $\mu \leq \frac{n}{3}$, this is asymptotically tight.
\end{restatable}

\paragraph{Useful invariants.}

As discussed in \Cref{sec:theory-spea:proofOutline}, our invariants mostly consider the lengths of intervals between the objective values in the current population of the \spea with $\mu \leq n$.
To this end, for each iteration $t \in \N$ of \Cref{alg:spea2}, let $(x^t_i)_{i \in [\mu]}$ denote the individuals from~$P_t$ (at the beginning of the \texttt{while} loop) in increasing order of their first \omm objective, that is, in increasing order of their number of ones.
Moreover, we assume $x^t_1 = 0^n$ and $x^t_\mu = 1^n$.
Then, for all $i \in [\mu - 1]$, we denote the length of interval~$i$ in iteration~$t$ by
\begin{equation}
  \label{eq:interval_definition}
  L^t_i \coloneqq \omm_1(x^t_{i + 1}) - \omm_1(x^t_ i) .
\end{equation}
In particular, we are interested in the minimum (and the maximum) length induced by~$P_t$, defined as
\begin{equation}
  \label{eq:minimum_length_definition}
  X_t \coloneqq \min\nolimits_{i \in [\mu - 1]} L^t_i \quad\textrm{ and }\quad
  Y_t \coloneqq \max\nolimits_{i \in [\mu - 1]} L^t_i,
\end{equation}
and their quantity, defined as
\begin{align}
  \label{eq:amount_minimum_length_definition}
  N_t & = |\{i \in [\mu - 1] \mid L^t_i = X_t\}| \textrm{ and} \\
  \notag
  M_t & = |\{i \in [\mu - 1] \mid L^t_i = Y_t\}|.
\end{align}

From now on, we focus on the steady-state \spea.
We first show that the minimum interval length never decreases.

\begin{restatable}{lemma}{increasingSmallestInterval}
  \label{lem:increasing_smallest_interval}
  Consider the steady-state \spea optimizing \omm with either 1-bit or standard bit mutation, $\mu \in [n]$, and the initial population containing~$0^n$ and~$1^n$.
  The sequence $(X_t)_{t \in \N}$ defined in \cref{eq:minimum_length_definition} is non-decreasing.
\end{restatable}

The next result shows that the minimum interval length increases during an iteration or the number of such intervals does not decrease and that the analogous statement holds for the maximum interval length, which does not increase.

\begin{restatable}{lemma}{increasingMinimumNumber}
  \label{lem:increasing_minimum_number}
  Consider the steady-state \spea optimizing \omm with 1-bit mutation, $\mu \in [n]$, and the initial population containing~$0^n$ and~$1^n$.
  The sequences $(-X_t, N_t)_{t \in \N}$ and $(Y_t, M_t)_{t \in \N}$ defined by \cref{eq:minimum_length_definition,eq:amount_minimum_length_definition} are each lexicographically non-increasing.
\end{restatable}

The following result shows that if we consider 1-bit mutation, either the offspring or its parent are removed once the intervals have a minimum length of at least two.

\begin{restatable}{lemma}{easyRemoval}
  \label{lem:easyRemoval}
  Consider the steady-state \spea optimizing \omm with 1-bit mutation, $\mu \in [n]$, and the initial population containing~$0^n$ and~$1^n$.
  Consider an iteration $t \in \N$ such that $X_t > 1$.
  Last, assume that mutation mutates~$x$ into~$y$.
  Then during the removal phase, we either remove~$x$ or~$y$.
\end{restatable}

With the next result, we show that intervals of minimum length at the borders of the interval $[0 .. n]$ do not move away from there if neither the minimum length nor the number of these intervals changes.

\begin{restatable}{lemma}{smallIntervalsStayAtTheBorders}
  \label{lem:small_intervals_stay_at_the_borders}
  Consider the steady-state \spea optimizing \omm with 1-bit mutation, $\mu \in [n]$, and the initial population containing~$0^n$ and~$1^n$.
  Recall \cref{eq:alpha_and_beta,eq:interval_definition,eq:minimum_length_definition,eq:amount_minimum_length_definition}.
  Consider an iteration $t \in \N$ such that $X_t > 1$, that $(-X_t, N_t) > (-\alpha, \beta)$,
  and that $(X_t, N_t) = (X_{t+1}, N_{t+1})$.
  Then, if $L_1^t = X_t$, it follows that $L_1^{t+1} = X_t$. Similarly, if $L_{\mu-1}^t = X_t$, then $L_{\mu-1}^{t+1} = X_t$.
\end{restatable}

Last for this part, we show that the \spea quickly reaches a state in which the minimum interval length is at least two.

\begin{restatable}{lemma}{minimumIntervalAtLeastTwo}
  \label{lem:minimum_interval_at_least_two}
  Consider the steady-state \spea optimizing \omm with 1-bit mutation, $\mu \in [n]$, and the initial population containing~$0^n$ and~$1^n$.
  Recall \cref{eq:minimum_length_definition}, and let $t \in \N$.
  If $X_t = 1$, then in an expected number $p$ of $O(\mu^2 n)$ iterations, we have $X_t < X_{t+p}$ and/or $Y_{t+p} \leq 3$.
\end{restatable}

\paragraph{Reducing the number of intervals of minimum length.}

We show that the minimum interval length is quickly increased as long as the \spea did not compute an optimal spread yet.
As outlined in \Cref{sec:theory-spea:proofOutline}, our proof relies on the fact that if an interval of minimum length and one whose length is at least two larger are next to each other, they can be combined into two new intervals whose length is each larger than the current minimum length.
To this end, it is important that such two intervals move to each other.
We show that this happens via a case distinction with respect to whether a minimum length interval is at either border of $[0 .. n]$.
A crucial building block is \Cref{lem:small_intervals_stay_at_the_borders}, which guarantees that intervals of minimum length remain at the border of $[0 .. n]$, once they are there.
We then show that once we have intervals of minimum length at both borders, then the minimum distance between a minimum length interval and one whose length is at least two larger does not increase.
Consequently, these intervals eventually move closer to each other until they create two new intervals of a length larger than the minimum.

Our main result following from the discussion above is the following bound on the number of iterations to increase the minimum interval length or the number of intervals thereof.

\begin{restatable}{lemma}{timeToIncreaseMinimumLength}
  \label{lem:time_to_increase_minimum_length}
  Consider the steady-state \spea optimizing \omm with 1-bit mutation, $\mu \in [n]$, and the initial population containing~$0^n$ and~$1^n$.
  Recall \cref{eq:alpha_and_beta,eq:minimum_length_definition}, and let $t \in \N$.
  If $X_t > 1$ and $(-X_t, N_t) > (-\alpha, \beta)$, then in an expected number $p$ of $O\bigl(\frac{\mu n \log(\mu)}{X_t}\bigr)$ iterations, we have $(-X_t, N_t) > (-X_{t+p}, N_{t+p})$.
\end{restatable}

\paragraph{Combining everything.}
We combine all of our prior arguments in order to prove \Cref{thm:spea_run_time_optimal_spread}.
We essentially rely on \Cref{lem:time_to_increase_minimum_length} to gradually increase the minimum interval length until it reaches the optimal value~$\alpha$ (\cref{eq:alpha_and_beta}).
However, this requires that the minimum interval length is at least~$2$.
By \Cref{lem:minimum_interval_at_least_two}, we achieve such a state quickly, but it requires us to restrict the population size~$\mu$ of the \spea to at most~$\frac{n}{3}$.

\begin{proof}[Proof of \Cref{thm:spea_run_time_optimal_spread}]
  We first wait until the population does not contain any duplicate objective values, which takes at most $O(\mu n \log n)$ iterations by \Cref{thm:time_to_distinct_archive}, and contains the individuals~$0^n$ and~$1^n$, which takes the same asymptotic amount of time by \Cref{thm:time_until_extreme_values}.
  Note that this means that the minimum interval length is at least~$1$.
  Afterward, we first wait for the first following iteration~$t^*$ such that $X_{t^*} > 1$ or $Y_{t^*} \leq 3$.
  By \Cref{lem:minimum_interval_at_least_two}, we have $t^* = O(\mu^2 n)$.

  If $Y_{t^*} \leq 3$, then, recalling \cref{eq:alpha_and_beta}, since $\alpha \geq 3$, then $X_{t^*} = 3$.
  Thus, $X_{t^*} = \alpha$ and $N_{t^*} = \beta$, as we have $n = \alpha\beta+(\alpha+1)(\mu-1-\beta)$.
  This concludes this case.

  It is left to consider the case $X_{t^*} > 1$.
  Let $t \in \N_{\geq t^*}$ such that $(-X_{t}, N_{t}) > (-\alpha, \beta)$, that is, we did not compute an optimal spread yet.
  Then by \Cref{lem:time_to_increase_minimum_length}, the expected number of iterations $p$ to achieve either $X_{t} < X_{t+p}$ or $(-X_{t+p}, N_{t+p}) = (-\alpha, \beta)$ is $O\bigl(\frac{\mu n \log(\mu)}{X_{t}}\bigr)$.
  Increasing~$X_t$ requires reducing~$N_t$ at most~$\mu$ times.
  Afterward, we sum over all possible values values of~$X_t$, from~$1$ to $\alpha \leq n$, resulting in the harmonic series and thus concluding the proof.
\end{proof}

\section{The \nsga Does Not Compute an Optimal Spread in Polynomial Time}
\label{sec:theory-nsga}

We show that the steady-state \nsga takes with high probability a super-polynomial time to compute an optimal spread when starting in an unfavorable state.
This is in stark contrast to the \spea, which always computes an optimal spread in a rather fast polynomial time (\Cref{thm:spea_run_time_optimal_spread}).
We note that situations similar to the one we discuss seem to be common, as evidenced in \cite[Figure~1]{ZhengD24approx}.

\begin{theorem}
  \label{thm:nsga-counterexample}
  Let $c \in \N_{\geq 2}$ be a constant in~$n$, and assume that~$16 c$ divides~$n$, and let $n' \coloneqq \frac{n}{c}$.
  Consider that the steady-state \nsga optimizes \omm with 1-bit mutation and with $N = n' + 1$.
  Apply the definition of \cref{eq:interval_definition} to the \nsga.
  Assume that for all $i \in [N - 1] \setminus \{\frac{1}{8} n', \frac{1}{4} n'\}$, we have $L^0_i = c$, and that we have $L^0_{n' / 8} = c + 1$ and $L^0_{n' / 4} = c - 1$.
  Then, with probability at least $1 - e^{-\Omega(n)}$, the steady-state \nsga does not compute an optimal spread within polynomial time.
\end{theorem}

\Cref{thm:nsga-counterexample} shows that although the steady-state \nsga has computed almost an optimal spread, with only two intervals not being optimal, the algorithm still takes with overwhelming probability a super-polynomial amount of time to compute an optimal spread.
This behavior is due to the fact that the removal of an individual in each iteration of the \nsga only takes into account the two closest neighbors (in objective values).
Thus, empty intervals (of objective values) that neighbor empty intervals of a length that differs by one from their own swap their lengths at random with each other.
Due to the shape of the search space of \omm featuring fewer individuals with objective values close to the extreme ones than in the center, this introduces a bias toward individuals with an equal number of zeros and ones during mutation, making it more likely to produce offspring closer to the center.
This introduces a bias into how neighboring intervals swap their lengths.

Our counterexamples feature intervals that require the offspring to be placed closer to the extreme objective values than to the center.
This implies that the intervals are more likely to drift apart from each other rather than getting closer.
Ultimately, this results in an at least super-polynomial runtime with overwhelming probability until intervals of different lengths meet, which is required for an optimal spread.

The \spea does not have this problem as the \sigDists take into consideration more than just the distance to the two closest neighbors in objective space.
It is this additional information that allows intervals whose lengths differ by at least two to get closer to each other and align their lengths better (which is the essence of the proof of \Cref{lem:time_to_increase_minimum_length}).

In order to prove \Cref{thm:nsga-counterexample}, we make use of the following theorem, which shows that processes that move, within a certain interval, in expectation away from a target state only reach their target with a probability exponentially unlikely in the length of the distance to cover.

\begin{theorem}[Negative drift {\protect\cite{Kotzing16,Krejca19}}]
  \label{thm:negative-drift}
  Let $(X_t)_{t \in \N}$ be random variables over~$\R$.
  Moreover, let $X_0 \leq 0$, let $b \in \R_{> 0}$, and let $T = \inf \{t \in \N \mid X_t \geq b\}$.
  Suppose that there are values $a \in \R_{\leq 0}$, $\gamma \in (0, b)$, and $\varepsilon \in \R_{< 0}$ such that for all $t \in \N$, we have
  \begin{enumerate}
    \item $\E[(X_{t + 1} - X_t) \cdot \mathbf{1}\{X_t \geq a, t < T\} \mid X_t] \leq \varepsilon \cdot \mathbf{1}\{X_t \geq a, t < T\}$, that
    \item $|X_t - X_{t + 1}| \cdot \mathbf{1}\{X_t \geq a, t < T\} < \gamma \cdot \mathbf{1}\{X_t \geq a, t < T\} + \mathbf{1}\{X_t < a, t < T\}$, and that
    \item $X_{t + 1} \cdot \mathbf{1}\{X_t \geq a, t < T\} \leq 0$.
  \end{enumerate}
  Then, for all $t \in \N$, we have $\Pr[T \leq t] \leq t^2 \exp\left(-\frac{b |\varepsilon|}{2 \gamma^2}\right)$.
\end{theorem}

\begin{proof}[Proof of \Cref{thm:nsga-counterexample}]
  Note that the only possible changes to the population are those that change where the intervals of length $c + 1$ or $c - 1$ are located as well as the change that averages the two intervals of lengths $c + 1$ and $c - 1$ to~$c$.
  In the latter case, the algorithm computes an optimal spread.
  However, to do so, it is necessary that those two intervals are next to each other.
  Before this is the case, they can only move their interval index by one by swapping their position with a neighboring interval of length~$c$.

  Let~$T$ be the first iteration such that the two intervals of lengths $c + 1$ and $c - 1$ are neighbors.
  We show that the probability that~$T$ is polynomial is at most $e^{-\Omega(n)}$, thus proving the claim.
  Hence, in the following, we always assume implicitly that we only consider iterations $t \in T$ such that $t < T$.\footnote{Formally, this requires multiplying each statement with the indicator variable of the event $\{t < T\}$, which we omit.}

  For all $t \in \N$, let~$X_t$ denote the index of the interval of length $c + 1$ at the beginning of iteration~$t$.
  (Note that we omit the case if no such interval exists, as we only consider iterations less than~$T$.)
  Analogously, let~$Y_t$ denote the index of the interval of length $c - 1$ at the beginning of iteration~$t$.
  Note that both~$X$ and~$Y$ change by at most~$1$ in each iteration.

  We first consider~$X$ and determine its expected change, aiming to apply \Cref{thm:negative-drift}.
  To this end, we only consider those iterations in which~$X$ actually changes (and iteration~$0$) and only while~$X$ is in $[\frac{1}{16} n' .. \frac{3}{16} n']$.
  Moreover, we assume that the interval of length $c - 1$ does not move.
  Let the resulting process be~$X'$, and let~$T_x$ be the first point in time $t \in \N$ of process~$X'$ such that $X' \geq \frac{3}{16} n'$.
  Note that~$X'_0 = \frac{2}{16} n'$ by assumption.
  Furthermore note that for all $t \in \N$, if $X'_t < \frac{3}{16} n'$, the same is true for~$X$.
  Hence, if $Y \geq \frac{3}{16} n'$ until at least~$T_x$, then not considering the interval of length $c - 1$ moving does not affect the statements about~$X'$ at all.
  We make use of this observation at the end of this proof.

  Since~$X'$ changes in each iteration and it always neighbors an interval of length~$c$ during this time (by assumption), there are only two cases to consider.
  Let $t \in \N$.
  First, $X'_t$ decreases if the steady-state \nsga produces an offspring with $c(X'_t - 1) + 1$ ones (from the parent with~$c (X'_t - 1)$ ones) and the offspring survives the selection (out of the parent and the offspring).
  This amounts to a probability of $\bigl((n - c (X'_t - 1)) / n\bigr) \cdot \frac{1}{2}$.
  Note that~$X'_t$ can decrease since we omit the case that it is the first interval.
  Second, $X'_t$ increases if an offspring with $c X'_t$ ones is created (from the parent with $c X'_t + 1$ ones) and the offspring is selected.
  This probability is $\bigl(c X'_t + 1 / n\bigr) \cdot \frac{1}{2}$.
  Let $d \coloneqq \bigl((n - c (X'_t - 1)) / n\bigr) \cdot \frac{1}{2} + \bigl((c X'_t + 1) / n\bigr) \cdot \frac{1}{2} \leq 1$.
  Since~$X'$ is conditional on~$X$ changing and since we only consider $X'_t \leq \frac{3}{16} n'$, we get
  \begin{align*}
     & \E[X'_{t + 1} - X'_t \mid X'_t]
    = -1 \cdot \frac{n - c (X'_t - 1)}{2 n d} + 1 \cdot \frac{c X'_t + 1}{2 n d} \\
     & \qquad= \frac{2 c X'_t + c + 1 - n}{2 n d}
    \leq -\frac{1}{2 d} \left(1 - \frac{3}{8} - \frac{c + 1}{n}\right)           \\
     & \qquad\leq -\frac{1}{2 d}
    \leq -\frac{1}{2} .
  \end{align*}

  We apply \Cref{thm:negative-drift} to $(X'_t - \frac{2}{16} n')_{t \in \N}$ with $a = -\frac{1}{16} n'$, $b = \frac{1}{16} n'$, $\gamma = 1$, and $\varepsilon = -\frac{1}{2}$.
  We get for all $t \in \N$ that
  \begin{equation*}
    \Pr[T_x \leq t]
    \leq t^2 \exp\Bigl(-\frac{(1 / 16) n' (1 / 2)}{2}\Bigr)
    = t^2 e^{-\Omega(n)} .
  \end{equation*}

  For~$Y$ we proceed analogously, defining~$Y'$ in the same way as~$X'$ but only consider iterations where~$Y$ is in $[\frac{3}{16} n' .. \frac{5}{16} n']$.
  Moreover, let~$T_y$ be the first point in time $t \in \N$ of process~$Y'$ such that $Y'_t \leq \frac{3}{16} n'$.
  Last, let $t \in \N$ (less than~$T_y$) and let $d' \coloneqq \bigl((n - c Y'_t) / n\bigr) \cdot \frac{1}{2} + \bigl((c (Y'_t - 1) + 1) / n\bigr) \cdot \frac{1}{2} \leq \frac{1}{2}$.
  For similar reasons as before, we get
  \begin{align*}
     & \E[Y'_t - Y'_{t + 1} \mid Y'_t]
    = -1 \cdot \frac{n - c Y'_t}{2 n d'} + 1 \cdot \frac{c (Y'_t - 1) + 1}{2 n d'} \\
     & \qquad= \frac{2 c Y'_t - c + 1 - n}{2 n d'}
    \leq -\frac{1}{2 d'} \left(1 - \frac{3}{8} - \frac{c - 1}{n}\right)            \\
     & \qquad\leq -\frac{1}{4 d'}
    \leq -\frac{1}{2} .
  \end{align*}

  By \Cref{thm:negative-drift} applied to $(\frac{4}{16} n' - Y'_t)_{t \in \N}$, we get for all $t \in \N$ that
  $   \Pr[T_y \leq t]
    \leq t^2 \exp(-\frac{(1 / 16) n' (1 / 2)}{2})
    = t^2 e^{-\Omega(n)} $.

  Overall, we see that $X' \geq \frac{3}{16} n'$ and $Y' \leq \frac{3}{16} n'$, which is a necessary condition for the steady-state \nsga to compute an optimal spread, occurs for any polynomial number of iterations only with probability $e^{-\Omega(n)}$.
  Note that since we consider the intersection of the events that $X' < \frac{3}{16} n'$ and $Y' > \frac{3}{16} n'$ for all iterations less than, respectively,~$T_x$ and~$T_y$, it does not matter that we did not consider the other interval of length not equal to~$c$ for the analysis of either process.
  Ultimately, since the original process accounts for even more iterations than those considered by~$X'$ and~$Y'$, the result follows.
\end{proof}

\section{Conclusion}
\label{sec:conclusion}

We showed rigorously that the steady-state variant of the \spea, which generates one offspring each iteration, efficiently computes a desired optimal spread of individuals on the Pareto front of the classic OneMinMax benchmark when the population size is smaller than the size of the Pareto front (\Cref{thm:spea_run_time_optimal_spread}).
Moreover, we proved in the same setting that the steady-state variant of the \nsga can be initialized such that, with overwhelming probability, it does not compute an optimal spread within polynomial time (\Cref{thm:nsga-counterexample}).
This difference in runtime performance is due to the \sigDists used in the \spea for choosing which individual to remove, in contrast to the crowding distance used by the \nsga, noting that the \sigDists are computationally more expensive to compute than the crowding distance.

Our current rigorous upper bound for the expected time of the \spea to compute an optimal spread is not tight and also differs from the expected time it takes the algorithm to find the extreme solutions of the Pareto front.
This may be due to our proof method, which operates in certain progress levels, only accounting for progress in the worst level, ignoring potential progress elsewhere.
Improving this method or showing that the result is actually tight is an interesting open problem.

Another related problem is to prove a similar expected runtime for the \spea when using standard bit mutation instead of 1-bit mutation.
The former allows to change individuals more drastically, making it more challenging to measure the progress that is being made within a single iteration.


\section*{Acknowledgments}
This research benefited from the support of the FMJH Program PGMO.

This work has profited from many deep discussions at the Dagstuhl Seminars \href{https://www.dagstuhl.de/23361}{23361 \emph{Multiobjective Optimization on a Budget}} and \href{https://www.dagstuhl.de/24271}{24271 \emph{Theory of Randomized Optimization Heuristics}}.

}

\bibliographystyle{named}
\bibliography{ich_master,alles_ea_master,rest}

\cleardoublepage
\appendix
\onecolumn
\section*{Appendix}

This part contains the proofs that were omitted in the main paper.
For the sake of convenience, we repeat each statement with the same number as it appears in the main paper.

\section*{The \spea Computes an Optimal Spread Efficiently}

\subsection*{No Duplicate Objective Values}

\increasingArchive*
\begin{proof}
  At the beginning of iteration~$t$, we have that~$P_{t+1}$ contains all the non-dominated solutions from $P_t \cup Q_t$, which encompasses at least~$P_t$, as it only contains Pareto-optimal individuals.
  Thus, $|f(P_t)| \leq |f(P_{t+1})|$ before we potentially modify~$P_{t + 1}$.
  Since~$|P_t| \geq \mu$, it follows that the \spea does not enter \cref{line:fillUp}, and nothing changes if $|P_{t + 1}| = \mu$.
  Hence, for the remainder, we consider the case $|P_{t + 1}| > \mu$, which means that we iteratively remove individuals from~$P_{t + 1}$.
  If we delete an individual~$x$ from $P_{t+1}$, we replace $P_{t+1}$ by $P_{t+1} \setminus \{x\}$ (we denote the new set after deletion as $P_{t+1}$ as well).

  If at every moment of iteration~$t$ we only remove from $P_{t+1}$ individuals~$x$ such that $\sigma^1_x = 0$, then we do not change the value of $|f(P_{t+1})|$, and thus, $|f(P_t)| \leq |f(P_{t+1})|$.

  If we remove at some moment an individual $x$ from $P_{t+1}$ such that $\sigma^1_x > 0$, we have for all the remaining individuals~$y$ in $P_{t+1}$ that $\sigma^1_y > 0$. Thus, $|f(P_{t+1})| = |P_{t+1}|$ at every subsequent moment. Since the removal process stops once $|P_{t + 1}| = \mu$ and since $|P_t| = \mu$ holds by construction of the algorithm, it follows that $|f(P_t)| \leq |P_t| = \mu = |P_{t + 1}| = |f(P_{t+1})|$ at the end of the iteration, concluding the proof.
\end{proof}

\timeToDistinctArchive*
\begin{proof}
  For the sake of brevity, let~$f$ denote \omm.
  For all $t \in \N$, we denote by $P_{t \rightarrow t+1}$ the probability of the event $|f_1(P_t)| < |f_1(P_{t+1})|$ given $|f_1(P_t)| < \mu$, i.e.,
  \begin{equation*}
    P_{t \rightarrow t+1} \coloneqq \Pr[|f_1(P_t)| < |f_1(P_{t+1})| \ | \ |f_1(P_t)| < \mu].
  \end{equation*}

  We continue via a case distinction with respect to the shape of the objective values in any iteration $t \in \N$ such that~$P_t$ contains duplicate objective values and does not contain the entire Pareto front yet.
  We note that we do not lose new objective values, by \Cref{lem:increasing_archive}, as long as the population still contains duplicates.

  \paragraph{Case 1:} There is a hole in the objective values of the population, that is, there exist $a,b \in \N$ such that $a<b-1$ and for all $x \in P_t$, we have $ f_1(x) \in [0..a]\cup[b..n]$.

  Let $x_a, x_b \in P_t$ be such that $f_1(x_a) = a$ and $f_1(x_b) = b$. The probability that a mutation applied to $x_a$ creates an offspring~$y$ with $f_1(y) = a+1$ is $\frac{n-a}{n}$ for 1-bit mutation, and for standard bit mutation, it is $(1-\frac{1}{n})^{n-1} \frac{n-a}{n} \geq \frac{n-a}{en}$. Likewise, the probability that $x_b$ mutates into a $y$ with $f_1(y) = b-1$ is $\frac{b}{n}$ for 1-bit mutation, and at least $\frac{b}{en}$ for standard bit mutation. Since $a \leq b$, we have $(n-a) + b \geq n$, meaning that for at least one of $x_a$ and $x_b$ from which a $y$ with $f_1(y) \in \{a+1, b-1\}$ is mutated, such probability is at least $\frac{1}{2e}$. The probability that this individual is chosen as parent at least once in this iteration is $1-(1-\frac{1}{\mu})^\lambda$. Thus, $P_{t\rightarrow t+1}\geq \frac{1}{2e}(1-(1-\frac{1}{\mu})^\lambda)$.

  \paragraph{Case 2:} The objective values are one consecutive chunk that does not contain the extreme values, that is, $f_1(P_t) = [a..b]$ with $a>0$ and $b<n$.

  Let $x_a, x_b \in P_t$ be such that $f_1(x_a) = a$ and $f_1(x_b) = b$. The probability that a mutation applied to $x_b$ creates an offspring $y$ with $f_1(y) = b+1$ is $\frac{n-b}{n}$ for 1-bit mutation, and for standard bit mutation, it is $(1 - \frac{1}{n})^{n-1} \frac{n-b}{n} \geq \frac{n-b}{en}$. Likewise, the probability that $x_a$ mutates into a $y$ with $f_1(y) = a-1$ is $\frac{a}{n}$ for 1-bit mutation, and at least $\frac{a}{en}$ for standard bit mutation. The probability that at least one of $x_a$ and $x_b$ gives an offspring $y$ with $f_1(y) \in \{a-1, b+1\}$ is at least $\frac{n-b+a}{2ne} = \frac{n-|f_1(P_t)|+1}{2ne}$. Thus $P_{t\rightarrow t+1}\geq \frac{n-|f_1(P_t)|+1}{2ne}(1-(1-\frac{1}{\mu})^\lambda)$.

  \paragraph{Case 3:} The objective values are one consecutive chunk that contains the left extreme value, that is, $f_1(P_t) = [0..b]$ with and $b<n$.

  Let $x_b \in P_t$ be such that $f_1(x_b) = b$. The probability that a mutation applied to $x_b$ creates an offspring $y$ with $f_1(y) = b+1$ is at least $\frac{n-b}{en} = \frac{n-|f_1(P_t)|+1}{en}$. Thus $P_{t\rightarrow t+1}\geq \frac{n-|f_1(P_t)|+1}{ne}(1-(1-\frac{1}{\mu})^\lambda)$.

  \paragraph{Case 4:} The objective values are one consecutive chunk that contains the right extreme value, that is, $f_1(P_t) = [a..n]$ with and $a>0$.

  Let $x_a \in P_t$ be such that $f_1(x_a) = a$. The probability that a mutation applied to $x_a$ creates an offspring $y$ with $f_1(y) = a-1$ is at least $\frac{a}{ne}= \frac{n-|f_1(P_t)|+1}{ne}$. Thus $P_{t\rightarrow t+1}\geq \frac{n-|f_1(P_t)|+1}{ne}(1-(1-\frac{1}{\mu})^\lambda)$.

  \paragraph{Conclusion.}
  The expected number of iterations until the population contains only distinct objective values or covers all objective values is at most the sum of the waiting times for each single improvement.
  That is, it is bounded from above by $O(\sum_{k=1}^{\mu}\frac{2ne}{(n-k+1)(1-(1-\frac{1}{\mu})^\lambda)})=O(\frac{n\log n }{1-(1-\frac{1}{\mu})^\lambda})$.
  Since we evaluate the objective function~$\lambda$ times each iteration and we have~$\mu$ evaluations initially, the result follows.
\end{proof}

\timeUntilExtremeValues*
\begin{proof}
  For the sake of brevity, let~$f$ denote \omm.
  We first prove the upper bound and then discuss the lower bound.
  We show that if $\max f_1(P_t) - \min f_1(P_t) \neq n$, then via mutation, the probability that $\max f_1(P_t) - \min f_1(P_t) < \max f_1(P_{t+1}) - \min f_1(P_{t+1})$ is at least $\frac{n - b + a}{2e n}(1 - (1 - \frac{1}{\mu})^\lambda)$.
  To this end, we note that we never remove the extreme objective values in the parent population, as they denote the boundaries of the convex hull spanned by all objective values of \omm, thus always having the largest distances to all other objective values.

  Denote by $a = \min f_1(P_t)$ and $b = \max f_1(P_t)$. Let the solutions $x_a, x_b$ be in $P_t$ such that $f_1(x_a) = a$ and $f_1(x_b) = b$. Using 1-bit mutation or standard bit mutation, the probability of obtaining a mutated solution $y$ from $x_a$ such that $f_1(y) = f_1(x_a) - 1$ is at least $\frac{a}{e n}$. Similarly, the probability of obtaining a mutated solution $y$ from $x_b$ such that $f_1(y) = f_1(x_b) + 1$ is $\frac{n - b}{e n}$.
  Thus, the probability that $\max f_1(P_t) - \min f_1(P_t) < \max f_1(P_{t+1}) - \min f_1(P_{t+1})$ is $\frac{n - b + a}{2 e n}$ if the correct parent is chosen.

  The probability of choosing a suitable parent for mutation as described above is at least $1 - (1 - \frac{1}{\mu})^\lambda$, satisfying the bound claimed at the beginning.
  Thus, in an expected number $p$ of iterations, where $p = O\left(\frac{n}{(n - b + a) (1 - (1 - \frac{1}{\mu})^\lambda)}\right)$, we have $\max f_1(P_t) - \min f_1(P_t) < \max f_1(P_{t + p}) - \min f_1(P_{t + p})$.
  Summing over all possible values from $b - a$, ranging from~$1$ to~$n$, results in $O\Big(\frac{n\log n}{1-(1-\frac{1}{\mu})^\lambda}\Big)$ expected iteration until we are done.
  Accounting for the initial~$\mu$ evaluations and the~$\lambda$ evaluations each iteration concludes the proof of the upper bound.

  For the lower bound, we only aim at finding the individual~$0^n$, which is necessary for the algorithm to find both~$0^n$ and~$1^n$.
  We first argue that each initial individual has at least~$\frac{5 n}{12}$ ones with overwhelming probability.
  Since for each such individual, the bits are determined uniformly at random, by a classic multiplicative Chernoff bound, it follows that the probability that individual $i \in [\mu]$ has fewer than~ $\frac{5 n}{12}$ ones is at most $\exp(-\Omega(n))$.
  Via a union bound over all~$\mu$ initial individuals, the probability that all of them have at least~$\frac{5 n}{12}$ ones is at least $1 - \mu \exp(-\Omega(n)) = 1 - o(1)$, since $\mu \leq \frac{n}{3}$.
  We condition in the following on this event, which results in a lower bound for the unconditional expected value in the same order, as $1 - o(1) \geq \frac{1}{2}$.

  For the remaining part, we first wait until an individual with at most~$\frac{n}{12}$ ones.
  Note that in order to have such an individual, at least~$\frac{n}{3}$ improvements need to be made from the individuals whose number of ones is initially at least~$\frac{5 n}{12}$, since the number of ones changes by exactly one during each mutation.
  By this point in time, by \Cref{lem:increasing_archive}, individuals with extreme objective values cannot have duplicates, since the algorithm visited at least~$\frac{n}{3}$ new objective values.

  Last, notice that the probability to make progress is, for similar reasons as discussed above, for the same notation of~$a$, exactly $(1 - (1 - \frac{1}{\mu})^\lambda) \frac{a}{n}$, resulting in an expected waiting time of $\frac{n}{a} \cdot (1 - (1 - \frac{1}{\mu})^\lambda)^{-1}$ to make progress of one.
  Since the algorithm still requires to go through all values from~$1$ to~$\frac{n}{12}$ for~$a$, the lower bound follows, concluding the proof.
\end{proof}

\subsection*{Useful Invariants}

\increasingSmallestInterval*
\begin{proof}
  For the sake of brevity, let~$f$ denote \omm, and omit the current iteration when clear $t \in \N$.
  Note that duplicates are always removed first, so we consider individuals with new objective values.

  First, let $i \in [\mu - 1]$, and consider an offspring~$y$ such that $f_1(x_i) < f_1(y) < f_1(x_{i+1})$, as this is the only case how the minimum interval length can be reduced.
  Then four cases can occur.

  \paragraph{Case $f_1(y) - f_1(x_i) \geq X_t$ and $f_1(x_{i+1}) - f_1(y) \geq X_t$.}
  In this case, $X_{t+1} \geq X_t$ no matter which solution gets removed.

  \paragraph{Case $f_1(y) - f_1(x_i) < X_t$ and $f_1(x_{i+1}) - f_1(y) \geq X_t$.}
  Here, we remove either $y$ or $x_i$. In both cases, $X_{t+1} \geq X_t$ because $f_1(x_{i+1}) - f_1(y) \geq X_t$, and if $i > 1$, $f_1(x_i) - f_1(x_{i-1}) \geq X_t$.

  \paragraph{Case $f_1(y) - f_1(x_i) \geq X_t$ and $f_1(x_{i+1}) - f_1(y) < X_t$.}
  Again, we remove either $y$ or $x_{i+1}$. In both cases, $X_{t+1} \geq X_t$ because $f_1(y) - f_1(x_i) \geq X_t$, and if $i+1 < \mu$, $f_1(x_{i+2}) - f_1(x_{i+1}) \geq X_t$.

  \paragraph{Case $f_1(y) - f_1(x_i) < X_t$ and $f_1(x_{i+1}) - f_1(y) < X_t$.}
  In this case, we remove $y$.

  \paragraph{Conclusion.}
  The minimum interval length is never reduced.

  Next, consider without loss of generality (due to the symmetry of zeros and ones) an offspring~$y$ such that $f_1(y) < f_1(x_1)$.
  If $f_1(x_1) - f_1(y) < X_t$, then $\sigma^1_y = \sigma^1_{x_1}$, but $\sigma^2_y > \sigma^2_{x_1}$, as~$y$ is an extreme value in the population.
  Thus,~$x_1$ is removed, and $f_1(x_2) - f_1(y) > f_1(x_2) - f_1(x_1) \geq X_t$, concluding the proof.
\end{proof}

\increasingMinimumNumber*
\begin{proof}
  For the sake of brevity, let~$f$ denote \omm, and omit the current iteration when clear $t \in \N$.

  We first show the claim for $(X_t, N_t)$.
  By \Cref{lem:increasing_smallest_interval}, we have $X_{t+1} \geq X_t$. Let us suppose that $X_{t+1} = X_t$.
  We aim to show that~$N_t$ does not increase.

  Without loss of generality, we assume that when creating the new population, we select the parent $x_i^t$ and generate, via 1-bit mutation, an individual $y$ such that $f_1(y) = f_1(x_i^t) + 1$.
  Three cases can occur in the reduction step.

  \paragraph{Case 1: Removing $y$.}
  In this case, nothing changes, and $N_t$ remains the same.

  \paragraph{Case 2: Removing $x_i^t$.}
  Here, we assume that $f_1(x_i^t) - f_1(x_{i-1}^t) \leq f_1(x_{i+1}^t) - f_1(y)$.
  With $x_i^{t+1} = y$, we obtain two new intervals of lengths $f_1(y) - f_1(x_{i-1}^t)$ and $f_1(x_{i+1}^t) - f_1(y)$.
  We have $f_1(y) - f_1(x_{i-1}^t) > f_1(x_i^t) - f_1(x_{i-1}^t) \geq X_t$, and $f_1(x_{i+1}^t) - f_1(y) \geq f_1(x_i^t) - f_1(x_{i-1}^t)$.
  This implies that if $f_1(x_{i+1}^t) - f_1(y) \leq X_t$, then
  \[
    f_1(x_i^t) - f_1(x_{i-1}^t) \leq X_t \quad \text{and} \quad f_1(x_{i+1}^t) - f_1(y) \leq X_t.
  \]

  Since we also have $f_1(x_i^t) - f_1(x_{i-1}^t) \geq X_t$, it follows that if $f_1(x_{i+1}^t) - f_1(y) \leq X_t$, then
  \[
    f_1(x_i^t) - f_1(x_{i-1}^t) = f_1(x_{i+1}^t) - f_1(y) = X_t.
  \]

  \paragraph{Case 3: Removing another solution $x_j^t$.}
  This implies that $X_t = 1$. Without loss of generality, assume $f_1(x_j^t) - f_1(x_{j-1}^t) = 1$.
  In this scenario, we create three intervals of lengths $f_1(y) - f_1(x_i^t) = 1$, $f_1(x_{i+1}^t) - f_1(y)$, and $f_1(x_{j+1}^t) - f_1(x_{j-1}^t)$.
  We remove intervals of lengths $f_1(x_{i+1}^t) - f_1(x_i^t)$, $f_1(x_{j+1}^t) - f_1(x_j^t)$, and $f_1(x_j^t) - f_1(x_{j-1}^t) = 1$.
  We assume that $f_1(x_{j+1}^t) - f_1(x_{j-1}^t) > 1$, and if $f_1(x_{i+1}^t) - f_1(y) = 1$, then $f_1(x_{j+1}^t) - f_1(x_j^t) = 1$ (otherwise, we would remove $y$ instead of $x_j^t$).

  This concludes the case for $(X_t, N_t)$.
  We continue with $(Y_t, M_t)$.

  After the mutation of some solution into a solution $y$ via 1-bit mutation, we are sure that there exists a solution $w$ such that $\sigma_w^1 = 1$ and $\sigma_w^2 < Y_t$. This implies that $Y_t \geq Y_{t+1}$.

  If $\sigma_w^2 < Y_t - 1$, we are sure that if $Y_t = Y_{t+1}$, then $M_t \geq M_{t+1}$.

  If $\sigma_w^2 \geq Y_t - 1$ for all $w$ such that $\sigma_w^1 = 1$, then after the mutation, the number of empty intervals of length $Y_t$ decreases. Furthermore, after removing a solution, we increase this number by one, i.e., $M_t = M_{t+1}$.
\end{proof}

\easyRemoval*
\begin{proof}
  For the sake of brevity, let~$f$ denote \omm, and let $i \in [\mu]$ such that $x = x_i$.
  Note that since $X_t >1$ is the case, we cannot have $\sigma_y^1 = 0$.
  Without loss of generality, assume that $x_i$ mutates into $y$ such that $f_1(x_i) = f_1(y) - 1$. Then, $\sigma_y^1 = \sigma_{x_i}^1 = 1$. The only other solution, $z$, that can have $\sigma_z^1 = 1$ is $x_{i+1}$ when $L_i^t = 2$.

  Thus, if $\sigma_z^1 \neq 1$, we remove either $y$ or $x_i$ because they have the smallest $\sigma^1$. If $\sigma_z^1 = 1$, we remove $y$ because $\sigma_y^2 = 1$, and it is the only solution to have $\sigma^2 = 1$.
\end{proof}

\smallIntervalsStayAtTheBorders*
\begin{proof}
  For the sake of brevity, let~$f$ denote \omm.
  Suppose that $L_1 = X_t$. If $x_1 = 0^n$ mutates into a $y$, then we remove $y$. If~$x_2$ mutates into $y$ such that $f_1(x_2) = f_1(y) + 1$, then we remove the new solution again because we cannot decrease $X_t$ by \Cref{lem:increasing_smallest_interval}.

  Now, suppose $x_2$ mutates into $y$ such that $f_1(x_2) = f_1(y) - 1$. We discuss three cases and apply \Cref{lem:easyRemoval}.
  \begin{enumerate}
    \item If $L_2^t = X_t$, we remove $y$ because it has the smallest $\sigma^2$ between the two.
    \item If $L_2^t  = X_t + 1$, then $y$ and $x_2$ have the same $\sigma^1$ and $\sigma^2$, and $\sigma_y^3 \leq \sigma_{x_2}^3$ but $\sigma_y^4 < \sigma_{x_2}^4$.
    \item If $L_2^t > X_t + 1$, we decrease $N_t$.
  \end{enumerate}
  If the mutated solution differs from both $x_1$ and $x_2$, then according to \Cref{lem:easyRemoval}, we do not remove $x_1$ and $x_2$, and we still have $L_1^{t+1} = X_t$.

  The case $L_{\mu-1}^t = X_t$ is analogous to the case but considering the the number of zeros instead of ones.
  This concludes the proof.
\end{proof}

\minimumIntervalAtLeastTwo*
\begin{proof}
  For the sake of brevity, let~$f$ denote \omm.
  Suppose that $X_t = 1$ and that $Y_t > 3$. Assume one of the following statements holds:
  \begin{enumerate}
    \item $L_1^t = 1$ and $L_2^t > 2$,
    \item $\exists i \in [2..\mu - 2] \colon L_i^t = 1, \, L_{i+1}^t > 2, \, L_{i-1}^t > 2$, or
    \item $L_{\mu - 1}^t = 1$ and $L_{\mu - 2}^t > 2$.
  \end{enumerate}
  For all these cases, we show that the probability of achieving $(-X_t, N_t) > (-X_{t+1}, N_{t+1})$ is at least $\frac{n-1}{2\mu n} \geq \frac{1}{4 \mu}$.
  \paragraph{Case (i).}
  If $L_1^t = 1$ and $L_2^t > 2$, the probability of achieving $(-X_t, N_t) > (-X_{t+1}, N_{t+1})$ is at least $\frac{n-1}{\mu n}$. This occurs if $x_2$ mutates into a solution $y$ such that $f_1(y) = f_1(x_2) + 1$.

  \paragraph{Case (ii).}
  If there is an $i \in [2..\mu - 2]$ such that $L_i^t = 1$, $L_{i+1}^t > 2$, and $L_{i-1}^t > 2$, then the probability that $x_i$ mutates into $y$ such that $f_1(y) = f_1(x_i) - 1$ is $\frac{f_1(x_i)}{\mu n}$, and the probability that $x_{i+1}$ mutates into $y$ such that $f_1(y) = f_1(x_{i+1}) + 1$ is $\frac{n-f_1(x_i)-1}{\mu n}$. Thus, the probability of achieving $(-X_t, N_t) > (-X_{t+1}, N_{t+1})$ is at least $\frac{n-1}{2 \mu n}$.

  \paragraph{Case (iii).}
  This case is analogous to case~(i) by considering the number of zeros instead of the number of ones.

  Suppose now that none of the above statements hold. In this case, there exists some $j$ such that $\sigma_{x_j}^1 = 1$ and $\sigma_{x_j}^2 = 2$. Let $i$ be such that $L_i^t = Y_t$. If $x_i$ mutates into a solution $y$ such that $f_1(y) = f_1(x_i) + 1$, or $x_{i+1}$ mutates into a solution $y$ such that $f_1(y) = f_1(x_{i+1}) - 1$, then $(Y_t, M_t)$ decreases. This happens because $x_j$ satisfies $\sigma_{x_j}^1 = 1$ and $\sigma_{x_j}^2 = 2$. The probability of this event is at least $\frac{n + Y_t}{2\mu n} \geq \frac{1}{2\mu }$.

  Therefore, in all cases, the probability of achieving at least $(-X_t, N_t) > (-X_{t+1}, N_{t+1})$ or $(Y_t, M_t) > (Y_{t+1}, M_{t+1})$ is at least $\frac{1}{4 \mu }$. This implies that in an expected number $p$ of $O(\mu)$ iterations, we have at least $(-X_t, N_t) > (-X_{t+p}, N_{t+p})$ or $(Y_t, M_t) > (Y_{t+p}, M_{t+p})$. Furthermore, since there are at most $\mu - 1$ distinct intervals, we have that in an expected number $p$ of $O(\mu^2)$ iterations, we achieve $X_t < X_{t+p}$ or $Y_t > Y_{t+p}$.
  In addition, since the maximum interval length takes on values in $[4 .. n]$ before we stop, in an expected number $p$ of $O(\mu^2 n)$ iterations, we achieve $X_t < X_{t+p}$ or $Y_{t+p} < 4$.
  This concludes the proof.
\end{proof}

\subsection*{Reducing the Number of Intervals of Minimum Length}

\timeToIncreaseMinimumLength*

As mentioned in the main paper, the proof of \Cref{lem:time_to_increase_minimum_length} relies on an intricate case distinction of the location of the intervals in the objective space.
More specifically, we are interested in two types of quantities that each measure the distance (measured in number of intervals) of a minimum length interval to somewhere else.
The first type measures the distance of a minimum length interval to the borders of $[0 .. n]$.
The second type measures the distance to an interval whose length is at least larger by two, provided that at one of the borders, we have an interval of minimum length.
Below, we use the subscript \emph{b} for the first time and \emph{i} for the second type.
For all $t \in \N$, we define, recalling \cref{eq:interval_definition,eq:minimum_length_definition},
\begin{align*}
  S^t_{b, 1} & \coloneqq \{i - 1 \mid i \in [\mu - 1] \land L^t_i = X_t \land \forall h \in [i - 1]\colon L^t_h = X_t + 1\},                                                                          \\
  S^t_{b, 2} & \coloneqq \{\mu-i-1 \mid i \in [\mu - 1] \land L_i^t =X_t \land \forall h \in [i+1..\mu-1] \colon L_h^t =X_t+1 \},                                                                     \\
  S^t_{i, 1} & \coloneqq \{j-(i+1)\mid i \in [\mu - 1] \land j \in [i + 1 .. \mu - 1] \land L_i^t = X_t \land L_j^t > X_t +1 \land \forall h \in [j-(i+1)] \colon L_{i+h}^t =X_t +1\}, \textrm{\,and} \\
  S^t_{i, 2} & \coloneqq \{j-(i+1) \mid i \in [\mu - 1] \land j \in [i + 1 .. \mu - 1] \land L_i^t > X_t+1 \land L_j^t = X_t \land \forall h \in [j-(i+1)] \colon L_{i+h}^t =X_t +1\}.
\end{align*}

Given these definitions, we construct the following sets for which we prove useful invariants below.
Let $t in \N$.
We define conditionally
\begin{itemize}
  \item if $L_1^t \neq X_t$ and $L_{\mu-1}^t\neq X_t$, then $A^t \coloneqq S^t_{b, 1} \cup S^t_{b, 2} \cup S^t_{i, 1} \cup S^t_{i, 2}$,
  \item if $L_1^t = X_t$ and $L_{\mu-1}^t\neq X_t$, then $B^t \coloneqq S^t_{b, 2} \cup S^t_{i, 1} \cup S^t_{i, 2}$,
  \item if $L_1^t \neq X_t$ and $L_{\mu-1}^t=X_t$, then $B^t \coloneqq S^t_{b, 1} \cup S^t_{i, 1} \cup S^t_{i, 2}$, and
  \item if $L_1^t = X_t$ and $L_{\mu-1}^t=X_t$, then $C^t \coloneqq S^t_{i, 1} \cup S^t_{i, 2}$.
\end{itemize}

In the following, we show that the minimum of the sets defined above does not increase.
This means that intervals of minimum length get closer to their desired target (which is the border of $[0 .. n]$ or an interval whose length is larger by at least two).
Moreover, we show that the \spea quickly transitions into a state that is more favorable for computing an optimal spread.
To this end, we have three separate lemmas, each of which considers one of the sets~$A$, $B$, and~$C$ above.

The first of these tree lemmas shows that while no interval of minimum length is at the borders, such intervals move closer to the border or approach an interval whose length is at least two larger.

\begin{lemma}
  \label{lem:a_sets}
  Consider the steady-state \spea optimizing \omm with 1-bit mutation and any parent population size.
  Recall \cref{eq:alpha_and_beta,eq:minimum_length_definition} and the definitions above, and let $t \in \N$.
  If $(-X_t, N_t) > (-\alpha, \beta)$, $X_t > 1$,  $L_1^t \neq X_t$, $L_{\mu-1}^t \neq X_t$, $(X_t, N_t) = (X_{t+1}, N_{t+1})$, $L_1^{t+1} \neq X_t$, and $L_{\mu-1}^{t+1} \neq X_t$, then
  \[
    \min A^t \geq \min A^{t+1}.
  \]

  Moreover, in an expected number $p$ of $O(\frac{\mu n \log(\mu)}{X_t+1})$ iterations, we have
  \begin{enumerate}
    \item $(-X_t, N_t) > (-X_{t+p}, N_{t+p})$ or,
    \item $(-X_t, N_t) = (-X_{t+p}, N_{t+p})$ and $\min A^{t+p} = 0$ or,
    \item $(-X_t, N_t) = (-X_{t+p}, N_{t+p})$ and exactly one of $L_1^{t+p}, L_{\mu-1}^{t+p}$ is equal to $X_t$.
  \end{enumerate}
\end{lemma}

\begin{proof}
  Suppose $(-X_t, N_t) > (-\alpha, \beta)$, $X_t > 1$,  $L_1^t \neq X_t$, $L_{\mu-1}^t \neq X_t$, and $(X_t, N_t) = (X_{t+1}, N_{t+1})$.
  We first prove that the minimum does not increase.
  Afterward, we show the bound on the number of iterations.

  \paragraph{Step 1: Proving $\min A^{t+1} \leq \min A^t$.}

  We make a case distinction with respect to the four sets $S^t_{b, 1}$, $S^t_{b, 2}$, $S^t_{i, 1}$, and $S^t_{i, 2}$ that~$A^t$ consists of.

  \paragraph{Case 1: Let $i \in [\mu - 1]$ such that $L_i^t = X_t$, $L_h^t = X_t + 1$ for all $h \in [1.. i-1]$, and $i-1 = \min A^t$.}
  We consider different cases of how the offspring in this iteration is created.

  If $x_j$ (for $j > i+1$) mutates into $y$, then, by \Cref{lem:easyRemoval}, either $y$ or $x_j$ gets removed. This does not affect $x_h$ where $h \leq i+1$. Therefore, as long as $L_{\mu-1}^{t+1} \neq X_t$, we have $\min A^{t+1} \leq \min A^t$.

  If $x_j$ (for $j < i$) mutates into $y$, then, by \Cref{lem:easyRemoval}, either $y$ or $x_j$ is removed. In this case, $y$ must be removed because for $j > 1$, we have $L_{j-1}^t = L_j^t = X_t + 1$, and thus $\sigma_y^1 = \sigma_{x_j}^1$ and $\sigma_y^2=X_t<X_t+1=\sigma_{x_j}^2$.

  If $x_i$ mutates into $y$ and $f_1(y) = f_1(x_i) + 1$, then $f_1(x_{i+1}) - f_1(y) = X_t - 1$, so $y$ is removed, because $X_{t+1}\geq X_t$ according to \Cref{lem:increasing_smallest_interval}.
  If $x_i$ mutates into $y$ and $f_1(y) = f_1(x_i) - 1$, we have $\sigma_y^1 = \sigma_{x_i}^1 = 1$, $\sigma_y^2 = \sigma_{x_i}^2 = X_t$, and $\sigma_y^3 = \sigma_{x_i}^3 = X_t+1$. To analyze the other $\sigma_j$, note that for $h \in [i+1.. (i+1) + (i-1) - 1]$, we must have $L_h^t \leq X_t + 1$. Otherwise, we would have $\min A^t < i-1$. This implies that we remove $x_i$ and keep $y$.
  At this point, we either achieve $L_1^{t+1} = X_t$, which allows us to proceed to the second step (working with the set $B^{t+1}$), or we obtain $\min A^{t+1} < \min A^t$.

  If $x_{i+1}$ mutates into $y$ and $f_1(y) = f_1(x_{i+1}) - 1$, then $y$ is removed, because if not, $X_t$ would decrease.
  If $x_{i+1}$ mutates into $y$ and $f_1(y) = f_1(x_{i+1}) + 1$, then again they have the same values of $\sigma^1$ and $\sigma^2$, and for $h \in [i+2..(i+1)+(i-1)-1], L_h^t \leq X_t+1$. If there exists $h \in [i+2..(i+1)+(i-1)-1]$ such that $ L_h^t < X_t+1$ then we remove $y$. If $h \in [i+2..(i+1)+(i-1)-1], L_h^t = X_t+1$ and $L^t_{(i+1)+(i-1)} > X_t+1$ then we keep $y$, and $\min A^{t+1}<\min A^t$, otherwise, we remove $y$.

  \paragraph{Case 2: Let $i \in [\mu - 1]$ such that $L_i^t = X_t$, $L_h^t = X_t + 1$ for all $h \in [i+1.. \mu - 1]$, and $\mu - i - 1 = \min A^t$.}

  This is symmetric to Case 1, and the same arguments apply.

  \paragraph{Case 3: The minimum is achieved between some $x_i$ and $x_j$ (for $i < j$), where $L_i^t = X_t$, $L_j^t > X_t + 1$, and $L_{i+h}^t = X_t + 1$ for all $h \in [1, j-(i+1)]$, and $j-(i+1) = \min A^t$.}

  We make again a case distinction with respect to how the offspring is produced.
  If $x_h$ such that $h<i$ or $j+1<h$ mutates into $y$:
  According to Lemma \ref{lem:easyRemoval}, either $y$ or $x_h$ be removed. Consequently, we have $\min A^{t+1} \leq \min A^t$.

  If $x_h$ such that $i+1<h<j$ mutates into $y$, then $\sigma_y^1=\sigma_{x_h}^1=1$, $\sigma_y^2=X_t$, and $\sigma_{x_h}^2=X_t+1$, thus we remove $y$.

  If $x_{i+1}$ mutates into $y$ and if $f_1(y) = f_1(x_{i+1}) - 1$, then $f_1(y) - f_1(x_i) = X_t - 1$, and we remove $y$.

  If $x_{i+1}$ mutates into $y$ and if $f_1(y) = f_1(x_{i+1}) + 1$, using the fact that $j - (i+1) = \min A^t$ as we did earlier, we keep $y$ and remove $x_{i+1}$, which results in $\min A^{t+1} < \min A^t$.

  Similarly if $x_j$ mutates into $y$ such that $f_1(y) = f_1(x_j)-1$, then we remove $y$, and if $y$ is such that $f_1(y) = f_1(x_j)+1$ ,we keep $y$ and remove $x_j$ and get $\min A^{t+1} < \min A^t$.

  If $x_i$ mutates into $y$ and if $f_1(y) = f_1(x_i)+1$, we remove $y$.

  If $x_i$ mutates into $y$ and if $f_1(y) = f_1(x_i)-1$, then $\sigma_y^1 = \sigma_{x_i}^1 = 1$ and $\sigma_y^2 = \sigma_{x_i}^2 = X_t$, and we keep $y$ only if
  \begin{enumerate}
    \item $i-2\leq j-(i+1) = \min A^t$ and $\forall h< i-1, L_{i-1}^t >X_t$ or,
    \item $i-2 > j-(i+1) = \min A^t$ and there exists $h\leq j-(i+1)$ such that $L^t_{i-1-h-1} >X_t+1$.
  \end{enumerate}
  In both of these cases, we have $\min A^{t+1}\leq\min A^t$.

  The last case is regarding $x_{j+1}$. If it mutates into $y$ such that $f_1(y) = f_1(x_{j+1})+1$, then $\min A^{t+1} = \min A^t$. If $y$ is such that $f_1(y) = f_1(x_{j+1})-1$, as in the first case where we mutate $x_{i+1}$, we use the fact that $j-(i+1) = \min A^t$ to show that even if we remove $x_{j+1}$, we get $\min A^{t+1}\leq\min A^t$.

  \paragraph{Case 4: The minimum is achieved between some $x_i$ and $x_j$ (for $i < j$), where $L_i^t > X_t + 1$, $L_j^t = X_t$, and $L_{i+h}^t = X_t + 1$ for all $h \in [1, j-(i+1)]$, and $j-(i+1) = \min A^t$.}

  This case is symmetric to Case~3.

  \paragraph{Step 2: Proving the bound on the expected number of iterations.}

  We proved that if we are in Case~1 and $x_i$ mutates into $y$ such that $f_1(y)=f_1(x_i)-1$, we get $\min A^{t+1}<\min A^t$ or $L_1^{t+1}=X_t$. The probability of this event is at least $\frac{(X_t+1)\min A^t}{\mu n}$ for 1-bit mutation. We get the same bound if $i$ is such that $L_i^t =X_t$, that for all $h \in [i+1,\mu-1]$, we have $ L_h^t = X_t + 1 $, and that $\mu-i-1 = \min A^t$ (Case~2).

  Since the Cases 3 and 4 are similar, we discuss the case where $i,j$ are such that $i<j$, $L_i^t = X_t, L_j^t > X_t +1 $, for all $h \in [1..j-(i+1)]$, we have $L_{i+h}^t =X_t +1$, and that $j-(i+1) = \min A^t$ (Case~3).
  We said that if $x_{i+1}$ mutates into $y$ such that $f_1(y)=f_1(x_{i+1})+1$ or $x_j$ mutates into $y$ such that $f_1(y) = f_1(x_j)+1$, then $\min A^{t+1}<\min A^t$.

  The probability that $x_{j}$ mutates into $y$ with $f_1(y) = f_1(x_{j}) + 1$ is $\frac{n - f_1(x_{j})}{n}$. Similarly, the probability that $x_{i+1}$ mutates into~$y$ with $f_1(y) = f_1(x_{i+1}) + 1$ is $\frac{n - f_1(x_{i+1})}{n}$.

  For at least one of $x_{j}$ and $x_{i+1}$, a~$y$ with $f_1(y) \in \{f_1(x_{j}) + 1, f_1(x_{i+1}) + 1\}$ is mutated. The combined probability is at least
  \[
    \frac{2n - f_1(x_{j}) - f_1(x_{i+1})}{2n} \geq \frac{f_1(x_{j}) - f_1(x_{i+1})}{2n}.
  \]
  Therefore, in all cases, if $L_1^{t+1}\neq X_t, L_{\mu-1}^{t+1}\neq X_t$, then the probability that $\min A^t$ decreases is at least $\frac{(X_t+1)\min A^t}{2\mu n}$.

  The fact that $\min A^t \leq \mu$ concludes the proof.
\end{proof}

The next lemma is very similar to \Cref{lem:a_sets} but considers that there is already exactly one interval at minimum length at the border of $[0 .. n]$.

\begin{lemma}
  \label{lem:b_sets}
  Consider the steady-state \spea optimizing \omm with 1-bit mutation and any parent population size.
  Recall \cref{eq:alpha_and_beta,eq:minimum_length_definition} and the definitions above, and let $t \in \N$.
  If $(-X_t, N_t) > (-\alpha, \beta)$, $X_t > 1$, either $L_1^t \neq X_t$ or $L_{\mu-1}^t \neq X_t$, $(X_t, N_t) = (X_{t+1}, N_{t+1})$, and either $L_1^{t+1} \neq X_t$ or $L_{\mu-1}^{t+1} \neq X_t$, then
  \[
    \min B^t \geq \min B^{t+1}.
  \]

  Moreover, in an expected number $p$ of $O(\frac{\mu n \log(\mu)}{X_t})$ iterations, we have
  \begin{enumerate}
    \item $(-X_t, N_t) > (-X_{t+p}, N_{t+p})$ or,
    \item $(-X_t, N_t) = (-X_{t+p}, N_{t+p})$ and $\min B^{t+p} = 0$ or,
    \item $(-X_t, N_t) = (-X_{t+p}, N_{t+p})$, $L_1^{t+p}=X_t$, and $L_{\mu-1}^{t+p} =X_t$.
  \end{enumerate}
\end{lemma}

\begin{proof}
  The proof follows analogously to the proof of \Cref{lem:a_sets}. Using \Cref{lem:small_intervals_stay_at_the_borders}, we ensure that, under the given conditions, if $L_1^t=X_t$, then $L_1^{t+1}=X_t$, and if $L_{\mu-1}^t=X_t$, then $L_{\mu-1}^{t+1}=X_t$.
\end{proof}

The last of the three lemmas considers the case where there is a minimum length interval at both borders of $[0 .. n]$.
As a result, the distance between a minimum length interval and an interval of length at least two larger needs to decrease over time.

\begin{lemma}
  \label{lem:c-sets}
  Consider the steady-state \spea optimizing \omm with 1-bit mutation and any parent population size.
  Recall \cref{eq:alpha_and_beta,eq:minimum_length_definition} and the definitions above, and let $t \in \N$.
  If $(-X_t, N_t) > (-\alpha, \beta)$, $X_t > 1$, $L_1^t = X_t$, $L_{\mu-1}^t = X_t$, and $(X_t, N_t) = (X_{t+1}, N_{t+1})$, then
  \[
    \min C^t \geq \min C^{t+1}.
  \]

  Moreover, in an expected number $p$ of $O(\frac{\mu n \log(\mu)}{X_t})$ iterations, we have
  \begin{enumerate}
    \item $(-X_t, N_t) > (-X_{t+p}, N_{t+p})$ or,
    \item $(-X_t, N_t) = (-X_{t+p}, N_{t+p})$ and $\min C^{t+p} = 0$.
  \end{enumerate}
\end{lemma}

\begin{proof}
  The proof follows analogously to the proof of \Cref{lem:a_sets,lem:b_sets}. Using Lemma \ref{lem:small_intervals_stay_at_the_borders}, we ensure that, under the given conditions, $L_1^t=L_{\mu-1}^t=X_t$ implies that $L_1^{t+1}=L_{\mu-1}^{t+1}=X_{t+1}$.
\end{proof}

Using \Cref{lem:a_sets,lem:b_sets,lem:c-sets}, we prove \Cref{lem:time_to_increase_minimum_length}.

\begin{proof}[Proof of \Cref{lem:time_to_increase_minimum_length}]
  According to \Cref{lem:a_sets,lem:b_sets,lem:c-sets}, in an expected number $q$ of $O\bigl(\frac{\mu n \log(\mu)}{X_t}\bigr)$ iterations, we have $(-X_t, N_t) > (-X_{t+q}, N_{t+q})$ or $(-X_t, N_t) = (-X_{t+q}, N_{t+q})$ and $\min C^{t+q} = 0$.

  If $(-X_t, N_t) = (-X_{t+q}, N_{t+q})$ and $\min C^{t+q} = 0$, then there exists an $i$ such that $L_i^{t+q} = X_t$ and
  \begin{enumerate}
    \item $L_{i+1}^{t+q} > X_t$ or
    \item $L_{i-1}^{t+q} > X_t$.
  \end{enumerate}
  Without loss of generality, suppose that $L_{i+1}^{t+q} > X_t$. If $x_{i+1}$ mutates into $y$ such that $f_1(y) = f_1(x_{i+1}) + 1$, we keep $y$, remove~$x_{i+1}$, and obtain $(-X_{t+q+1}, N_{t+q+1}) < (-X_{t+q}, N_{t+q})$. The probability of this event is at least $\frac{1}{\mu n}$. If this does not occur, then we have $\min C^{t+q+1} = 0$ until we decrease $(-X_h, N_h)_{h \geq t+q}$.

  Thus, the expected number $q'$ to get $(-X_{t+q+q'}, N_{t+q+q'}) < (-X_{t+q}, N_{t+q})$ is $O(\mu n)$, which is less than the initial period of $O\bigl(\frac{\mu n \log(\mu)}{X_t}\bigr)$.
\end{proof}

\end{document}